\documentclass{article}

\usepackage{microtype}
\usepackage{graphicx}
\usepackage{subfigure}
\usepackage{booktabs} 
\usepackage{times}
\usepackage{epsfig}
\usepackage{amsmath}
\usepackage{amssymb}
\usepackage{bm}
\usepackage{amsthm}
\usepackage{multirow}
\usepackage{caption}
\usepackage[super]{nth}

\usepackage{amssymb}%
\usepackage{mathrsfs}%
\usepackage{nicefrac}
\usepackage{color,xcolor}
\usepackage{pifont}
\usepackage{array}
\usepackage{wrapfig}
\usepackage{mathtools}
\usepackage{colortbl}
\usepackage{mdframed}
\usepackage[utf8]{inputenc}
\usepackage{hyperref}
\usepackage{cleveref}

\newtheorem{lemma}{Lemma}
\newtheorem{theorem}{Theorem}
\newtheorem{corollary}{Corollary}

\newcommand{\RomanNumeralCaps}[1]
    {\MakeUppercase{\romannumeral #1}}
\usepackage[accepted]{icml2018}

\icmltitlerunning{}

\begin{document}

\twocolumn[
\icmltitle{\texorpdfstring{$O(N^{2})$}{TEXT} Universal Antisymmetry in Fermionic Neural Networks}

\icmlsetsymbol{equal}{*}

\begin{icmlauthorlist}
\icmlauthor{Tianyu Pang}{to1}
\icmlauthor{Shuicheng Yan}{to1}
\icmlauthor{Min Lin}{to1}
\end{icmlauthorlist}

\icmlaffiliation{to1}{Sea AI Lab, Singapore}

\icmlcorrespondingauthor{Tianyu Pang}{tianyupang@sea.com}
\icmlcorrespondingauthor{Min Lin}{linmin@sea.com}

\icmlkeywords{Machine Learning, ICML}

\vskip 0.3in
]

\printAffiliationsAndNotice{}  

\begin{abstract}
  Fermionic neural network (FermiNet) is a recently proposed wavefunction Ansatz, which is used in variational Monte Carlo (VMC) methods to solve the many-electron Schr\"{o}dinger equation. FermiNet proposes permutation-equivariant architectures, on which a Slater determinant is applied to induce antisymmetry. FermiNet is proved to have \emph{universal} approximation capability with a single determinant, namely, it suffices to represent any antisymmetric function given sufficient parameters. However, the asymptotic computational bottleneck comes from the Slater determinant, which scales with $O(N^3)$ for $N$ electrons. In this paper, we substitute the Slater determinant with a pairwise antisymmetry construction, which is easy to implement and can reduce the computational cost to $O(N^2)$. We formally prove that the pairwise construction built upon permutation-equivariant architectures can \emph{universally} represent any antisymmetric function. Besides, this universality can be achieved via \emph{continuous} approximators when we aim to represent ground-state wavefunctions.
\end{abstract}

\section{Introduction}
How to efficiently solve fermionic quantum many-body systems has been a crucial problem in quantum physics and chemistry, as well as in a large number of applications in material science and drug discovery~\citep{keimer2017physics,heifetz2020quantum}. On the one hand, density functional theory (DFT)~\citep{hohenberg1964inhomogeneous,becke1992density} is a widely used computational method in these fields due to its computational efficiency. There are also efforts trying to approximate DFT via deep neural networks~\citep{gilmer2017neural,schutt2019unifying}. However, due to the simplified assumptions on the wavefunction, i.e. limiting the wavefunction to be Slater determinant constructed out of orthogonal single particle wavefunctions, DFT relies on the exchange and correlation functional. Only approximate forms of the functionals (such as local-density approximations (LDA) and generalized gradient approximations (GGA)) are used in practice, leading to a potentially low prediction accuracy~\citep{zhao1999local,iikura2001long}.

On the other hand, quantum Monte Carlo (QMC) is a family of Monte Carlo methods that provide more accurate approximations. Especially, variational Monte Carlo (VMC) involves a trial wavefunction to minimize the energy. The accuracy of VMC relies on how close the trial wavefunction is to the ground state. Optimizing the parameters of the wavefunction Ansatz only depends on the data sampled by Markov chain Monte Carlo (MCMC) methods~\citep{hastings1970monte}, and thus does not need supervised data.

\citet{carleo2017solving} introduce neural networks to solve the quantum many-body problem by constructing neural quantum states as the Ansatz. \citet{pfau2020ab} develop Fermionic Neural Network (FermiNet) as a wavefunction Ansatz to be used in VMC. The key insight from FermiNet is that the single-electron wavefunctions can be replaced with a permutation-equivariant neural network to capture the correlation of particles, while the Slater determinant is merely an antisymmetrizer. The model family of FermiNet is proved to be universal for representing any antisymmetric function even with a single Slater determinant. Compared to FermiNet, \citet{hermann2020deep} design PauliNet which encodes built-in physical constraints to facilitate optimization. Along with this routine, graph neural networks (GNNs) and transformers are involved to model new Ansatz with equivariance~\citep{gao2021ab,atz2021geometric,tholke2022equivariant}.

A follow up work \citet{spencer2020better} implements FermiNet in JAX, which improves six fold compared to the original TensorFlow code. It shows the potential of applying recent advancements in both software and hardware to quantum chemistry problems. However, the antisymmetry of FermiNet is still implemented via the Slater determinant, which requires $O(N^{3})$ computation for $N$-electron systems. The cost of computing the Slater determinant (or Pfaffian) will asymptotically overwhelm the other parts of the model. It limits the scalability of FermiNet and other variants that uses the Slater determinant.

To this end, several efforts have been devoted to determinant-free strategies for separately modelling the sign and amplitude of wave functions~\citep{torlai2018neural,choo2020fermionic,szabo2020neural,stokes2020phases,inui2021determinant}. Nevertheless, they are applied to spin lattice systems or molecules in second quantization, as discussed in \citet{schatzle2021convergence}. In real-space cases, \citet{han2019solving} design a pairwise $O(N^2)$ Ansatz, but we demonstrate that their Ansatz requires discontinuous approximators to represent ground-state wavefunctions (detailed in Section~\ref{sec31}).


In this paper, we develop $O(N^{2})$ pairwise Ansatz based on the permutation-equivariant multi-electron functions used in FermiNet. Our pairwise Ansatz is straightforward to implement, with minor modification on the FermiNet architecture, and can be applied to other network-based architectures like PauliNet~\citep{hermann2020deep}. We formally prove that our pairwise Ansatz can \emph{universal} represent any ground-state wavefunction with \emph{continuous} approximators, which can be modelled via finite-capacity neural networks.



\section{Preliminary}
\label{Preliminary}
We let $x=(x_{1},\cdots,x_{N})\in\mathbb{R}^{dN}$ denote the $d$-dimensional coordinates of an $N$-electron system, where $x_{i}\in\mathbb{R}^{d}$ for $i=1,\cdots,N$. Given a set of single-electron orbitals $\{\phi_{1},\cdots,\phi_{N}\}$, we can construct an antisymmetric wavefunction Ansatz via the Slater determinant as 
\begin{equation}
    \psi_{\textrm{single}}(x)=\begin{vmatrix}\phi_{1}(x_{1}) & \cdots & \phi_{1}(x_{N})\\
    \vdots & & \vdots \\
    \phi_{N}(x_{1}) & \cdots & \phi_{N}(x_{N})\end{vmatrix}=\textrm{det}[\phi_{i}(x_{j})]\textrm{.}
\end{equation}
It is known that only a small subset of all possible antisymmetric functions can be written as the form of $\psi_{\textrm{single}}(x)$, i.e., the model family of $\psi_{\textrm{single}}(x)$ is \emph{not} universal. 

To this end, FermiNet~\citep{pfau2020ab} replaces the single-electron orbitals $\phi_{i}(x_{j})$ with the multi-electron functions $\phi_{i}(x_{j};\{x_{\setminus j}\})$ modeled by neural networks, where $\{x_{\setminus j}\}$ denotes the set of all electrons except $x_{j}$. By architecture design, the functions $\phi_{i}(x_{j};\{x_{\setminus j}\})$ are invariant to the permutation of elements in $\{x_{\setminus j}\}$. The wavefunction Ansatz of FermiNet can be written as
\begin{equation}
    \psi_{\textrm{Fermi}}(x)=\textrm{det}\left[\phi_{i}(x_{j};\{x_{\setminus j}\})\right]\textrm{.}
    \label{equ11}
\end{equation}
The model family of $\psi_{\textrm{Fermi}}(x)$ is proved to be universal to represent any antisymmetric function, as stated below:
\begin{lemma}
(Universality by \citet{pfau2020ab}) For any antisymmetric function $\Psi(x)$, there exist multi-electron functions $\phi_{1},\cdots\phi_{N}$, such that $\forall x$, there is $\psi_{\textrm{Fermi}}(x)=\Psi(x)$.
\end{lemma}

In practice, FermiNet uses multiple $\psi_{\textrm{Fermi}}^{k}(x)$ to achieve high accuracy for $k=1,\cdots,K$, and the ensemble wavefunction $\Psi_{\textrm{Fermi}}(x)$ is constructed as $\Psi_{\textrm{Fermi}}(x)=\sum_{k}\omega_{k}\psi_{\textrm{Fermi}}^{k}(x)$ where $\omega_{k}$ is the weight of $k$-th component.

\section{\texorpdfstring{$O(N^{2})$}{TEXT} Universal Antisymmetry}
Evaluating the Slater determinant operator in Eq.~\eqref{equ11} requires $O(N^{3})$ computation, which asymptotically limits the scalability of FermiNet and most of the other VMC methods. Given this, we develop an easy-to-implement $O(N^{2})$ substitute for the Slater determinant to build antisymmetric functions on top of a permutation-equivariant function. Technically, we require a function $F(x_i,x_j;\{x_{\setminus\{i,j\}}\})$ that is invariant to any permutation on the set $\{x_{\setminus\{i,j\}}\}$. Then we construct the \textbf{pairwise wavefunction Ansatz} as
\begin{equation}
    \psi_{\textrm{pair}}(x)=\!\!\prod_{1\leq i < j \leq N}\!\! \mathcal{A}\circ F(x_i,x_j;\{x_{\setminus\{i,j\}}\})\textrm{,}
    \label{equ34}
\end{equation}
where $\mathcal{A}$ is an antisymmetrizer such that $\mathcal{A}\circ F(x_{i},x_{j})=-\mathcal{A}\circ F(x_{j},x_{i})$ holds.\footnote{For notation compactness, here we omit the dependence of $F(x_i,x_j)$ on other elements in $\{x_{\setminus\{i,j\}}\}$ without ambiguity.} Now we demonstrate that $\psi_{\textrm{pair}}(x)$ is also antisymmetric:
\begin{lemma}
(Antisymmetry) $\psi_{\textrm{pair}}(x)$ is antisymmetric under the permutation of any two elements $x_{m}$ and $x_{n}$ in $x$.
\label{theorem1}
\end{lemma}
\begin{proof}
Assuming that $m<n$ without loss of generality, after we permute $x_{m}$ and $x_{n}$, the term $\mathcal{A}\circ F(x_m,x_i)$ changes to $\mathcal{A}\circ F(x_i,x_m)$ and the term $\mathcal{A}\circ F(x_i,x_n)$ changes to $\mathcal{A}\circ F(x_n,x_i)$ for $i\in(m,n)$. Besides, the term $\mathcal{A}\circ F(x_m,x_n)$ changes to $\mathcal{A}\circ F(x_n,x_m)$. Note that each term change contributes to a sign change (i.e., multiplying $-1$). So the total number of sign changes is $1+2\times(n-m-1)$, which is always an odd number and thus $\psi_{\textrm{pair}}(x)$ is antisymmetric.
\end{proof}

A simple implementation of $\mathcal{A}\circ F$ is $\mathcal{A}\circ F(x_{i},x_{j})=F(x_{i},x_{j})-F(x_{j},x_{i})$, which requires $O(1)$ extra computation on top of computing $F$. Since there are $C_{N}^{2}=N\times(N-1)/2$ terms in $\psi_{\textrm{pair}}(x)$, totally it requires $O(N^2)$ extra computation to induce antisymmetry. In the following, we analyse how different choices of $\mathcal{A}\circ F$ influence the \emph{universality} and \emph{continuity} of $\psi_{\textrm{pair}}(x)$.

\subsection{Instantiation from \citet{han2019solving}}
\label{sec31}
\citet{han2019solving} developed an $O(N^{2})$ scaling Ansatz $\psi_{\textrm{Han}}(x)$ formulated as
\begin{equation}
     \psi_{\textrm{Han}}(x)=\phi_{C}(x)\cdot\!\!\!\!\!\!\prod_{1\leq i<j\leq N}\!\!\!\!\!\!\left(\phi_{B}(x_{j},x_{i})-\phi_{B}(x_{i},x_{j})\right)\textrm{.}
    \label{equ30}
\end{equation}
Here $\phi_B(x_{j},x_{i})$ is a two-electron function, and $\phi_{C}(x)=\phi_{C}(x_{1},\cdots,x_{N})$ is a multi-electron symmetric function. The form of $\psi_{\textrm{Han}}(x)$ can be regarded as a special instantiation of $\psi_{\textrm{pair}}(x)$, where
\begin{equation*}
    \mathcal{A}\circ F(x_{i},x_{j})=\left(\phi_{B}(x_{j},x_{i})-\phi_{B}(x_{i},x_{j})\right)\cdot\phi_{C}(x)^{\frac{2}{N(N-1)}}\textrm{.}
\end{equation*}

We can show that for ground-state wavefunctions $\Psi^{g}(x)$ of fermionic systems,\footnote{We assume that the ground-state $\Psi^{g}(x)$ is continuous on $\mathbb{R}^{dN}$.} we \emph{cannot} find continuous constructions of both $\phi_{B}$ and $\phi_{C}$ to satisfy $\psi_{\textrm{Han}}(x)=\Psi^{g}(x)$, which makes it non-trivial for neural networks to approximate $\phi_{B}$ and $\phi_{C}$. To prove this, we resort to the tiling property of ground-state antisymmetric wavefunctions, as stated below:
\begin{lemma}
(Tiling property by \citet{ceperley1991fermion}) We define $\mathcal{N}(\Psi^{g})=\{x|\Psi^{g}(x)=0\}$ to be the node set of ground-state $\Psi^{g}$, and $\Omega(\Psi^{g},x)$ be the nodal cell around the point $x$. Then the tiling property claims that $\forall x \not\in \mathcal{N}(\Psi^{g})$,
\begin{equation}
    \bigcup_{\pi}\Omega(\Psi^{g},x_{\pi})=\mathbb{R}^{dN}\setminus \mathcal{N}(\Psi^{g})\textrm{,}
    \label{equ39}
\end{equation}
where $x_{\pi}$ is the input $x$ permuted by $\pi$ and the union $\bigcup_{\pi}$ is taken over all possible permutations.
\end{lemma}

According to the tiling property of ground-state wavefunctions, we can further prove the following:

\begin{theorem}
(Discontinuity) There exist ground-state wavefunctions $\Psi^{g}(x)$, such that if there is $\forall x$, $\psi_{\textrm{Han}}(x)=\Psi^{g}(x)$, then either $\phi_B$ or $\phi_C$ must be discontinuous on $\mathbb{R}^{dN}$. 
\end{theorem}
\begin{proof}
For notation compactness, we denote
\begin{equation}
    A(x;\phi_{B})=\prod_{i<j}\left(\phi_{B}(x_{j},x_{i})-\phi_{B}(x_{i},x_{j})\right)\textrm{,}
\end{equation}
and thus $\psi_{\textrm{Han}}(x)=A(x;\phi_{B})\cdot\phi_{C}(x)$. We first prove that there exist ground-state $\Psi^{g}(x)$, such that for any $\phi_B$, the sign of $A(x;\phi_{B})$ \emph{cannot} keep aligned with the sign of $\Psi^{g}(x)$. Specifically, we consider a three-electron system, where we have four coordinate candidates $x_1,x_2,x_3,x_4$. For simplicity, we write $A_{i,j,k}=A(x_i,x_j,x_k;\phi_{B})$ and $a_{i,j}=\phi_{B}(x_{j},x_{i})-\phi_{B}(x_{i},x_{j})$, where $A_{i,j,k}=a_{i,j}\cdot a_{i,k}\cdot a_{j,k}$. Now we can derive that
\begin{equation*}
    \begin{split}
        &A_{1,2,3}\cdot A_{1,2,4}\cdot A_{1,3,4}\cdot A_{2,3,4}\\
        =&a_{1,2}^{2}\cdot a_{1,3}^{2}\cdot a_{1,4}^{2}\cdot a_{2,3}^{2}\cdot a_{2,4}^{2}\cdot a_{3,4}^{2}\geq 0\textrm{,}
    \end{split}
\end{equation*}
which holds for any $x_1,x_2,x_3,x_4$. Thus, given any $\Psi^{g}(x)$ satisfying that $\textrm{sgn}(\Psi^{g}_{1,2,3}\cdot\Psi^{g}_{1,2,4}\cdot\Psi^{g}_{1,3,4}\cdot\Psi^{g}_{2,3,4})=-1$,\footnote{Intuitively, we can select $x_1$ and $x_2$ to be very close, and select $x_3$ and $x_4$ to satisfy $\textrm{sgn}(\Psi^{g}_{1,2,3}\cdot\Psi^{g}_{1,2,4})=-1$. Then due to the continuity of $\Psi^{g}$, there is $\textrm{sgn}(\Psi^{g}_{1,3,4}\cdot\Psi^{g}_{2,3,4})=1$.} there must be $\textrm{sgn}(A(x;\phi_{B}))\not=\textrm{sgn}(\Psi^{g}(x))$ for any $\phi_B$. Since $\textrm{sgn}((-\Psi^{g}_{1,2,3})\cdot(-\Psi^{g}_{1,2,4})\cdot(-\Psi^{g}_{1,3,4})\cdot(-\Psi^{g}_{2,3,4}))=-1$, there is also $\textrm{sgn}(A(x;\phi_{B}))\not=\textrm{sgn}(-\Psi^{g}(x))$.

Let $\Psi^{g}$ be any one of the ground-state functions whose sign cannot be represented via $A(x;\phi_{B})$, i.e., there exist $x'\not\in\mathcal{N}(\Psi^{g})$ such that $\textrm{sgn}(A(x';\phi_{B}))\not=\textrm{sgn}(\Psi^{g}(x'))$. Now we prove that if $\phi_{B}$ is continuous, there must be
\begin{equation}
    \mathcal{N}(A(x;\phi_{B}))\not\subset\mathcal{N}(\Psi^{g})\textrm{.}
\end{equation}
Specifically, we know that $A(x;\phi_{B})$ is continuous and $\textrm{sgn}(A(x';\phi_{B}))=\textrm{sgn}(-\Psi^{g}(x'))$. Assume, to the contrary, that $\mathcal{N}(A(x;\phi_{B}))\subset\mathcal{N}(\Psi^{g})$. Then according to the tiling property, we have the nodal cell $\Omega(\Psi^{g},x')\subset \mathbb{R}^{dN}\setminus \mathcal{N}(A(x;\phi_{B}))$, and due to the continuity of $A(x;\phi_{B})$, we know that the signs of both $A(x;\phi_{B})$ and $\Psi^{g}$ are unchanged in $\Omega(\Psi^{g},x')$. Thus, $\forall x\in\Omega(\Psi^{g},x')$, $\textrm{sgn}(A(x;\phi_{B}))=\textrm{sgn}(-\Psi^{g}(x))$. Since both $A(x;\phi_{B})$ and $\Psi^{g}(x)$ are antisymmetric, given any permutation $\pi$, there is $\forall x\in\Omega(\Psi^{g},x'_{\pi})$, $\textrm{sgn}(A(x;\phi_{B}))=\textrm{sgn}(-\Psi^{g}(x))$. Taking in the tiling property of Eq.~(\ref{equ39}), we achieve that $\forall x\not\in\mathcal{N}(\Psi^{g})$, there is $\textrm{sgn}(A(x;\phi_{B}))=\textrm{sgn}(-\Psi^{g}(x))$, which lead to contradiction with our three-electron counter-example above.

Now we have proved that if $\phi_{B}$ is continuous, there must be $\mathcal{N}(A(x;\phi_{B}))\not\subset\mathcal{N}(\Psi^{g})$, i.e., there exist an input $x''$ such that $A(x'';\phi_{B})=0$ and $\Psi^{g}(x'')\not=0$. Assume that $\psi_{\textrm{Han}}(x)=A(x;\phi_{B})\cdot\phi_{C}(x)=\Psi^{g}(x)$ for $\forall x$, then for any finite value of $\phi_{C}(x'')$, we always have $\psi_{\textrm{Han}}(x'')=A(x'';\phi_{B})\cdot\phi_{C}(x'')\not=\Psi^{g}(x')$, which leads to a contradiction. Thus, either $\phi_B$ or $\phi_C$ is discontinuous on $\mathbb{R}^{dN}$.
\end{proof}

\subsection{Achieving Universality with Continuous \texorpdfstring{$\phi_B$}{TEXT}}
In the above section, we show that there exist ground-state functions $\Psi^{g}(x)$ that $\psi_{\textrm{Han}}(x)$ have to represent via discontinuous $\phi_B$ or $\phi_C$, which makes finite-capacity neural networks difficult to approximate them, as also discussed in \citet{pfau2020ab}. Intuitively, this discontinuity comes from the two-electron dependence of $\phi_B$, i.e., $\phi_B$ cannot model the relation among multiple electrons.

To this end, we resort to the permutation-equivariant architecture designed in FermiNet, which includes $N$ multi-electron functions $\phi_{1}(x_{j};\{x_{\setminus j}\}),\cdots,\phi_{N}(x_{j};\{x_{\setminus j}\})$ that are invariant to any permutation in the set $\{x_{\setminus j}\}$, as described in Section~\ref{Preliminary}. In contrast, here we only need one function $\phi_{B}(x_{j};\{x_{\setminus j}\})$, and we instantiate $\psi'_{\textrm{pair}}(x)$ as
\begin{equation}
    \psi'_{\textrm{pair}}(x)=\!\!\!\!\!\!\prod_{1\leq i<j\leq N}\!\!\!\!\!\!\left(\phi_{B}(x_{j};\{x_{\setminus j}\})-\phi_{B}(x_{i};\{x_{\setminus i}\})\right)\textrm{.}
    \label{equ50}
\end{equation}
Now we prove that the model family of $\psi'_{\textrm{pair}}(x)$ is \emph{universal} to represent any ground-state wavefunction $\Psi^{g}$ using \emph{continuous} $\phi_{B}$, as stated below:
\begin{theorem}
(Continuous universality) For any ground-state wavefunction $\Psi^g(x)$, there exist a continuous multi-electron function $\phi_{B}$, such that $\forall x$, $\psi'_{\textrm{pair}}(x)=\Psi^{g}(x)$.
\label{theorem5}
\end{theorem}
\begin{proof}
For any ground-state wavefunction $\Psi^g(x)$, according to its tiling property, we can select one of the nodal cell $\Omega^{*}$, such that $\forall x\in\Omega^{*}$, there is $\Psi^g(x)>0$. Besides, there exist a permutation rule $\pi^{*}$ to map the points from other nodal cells into the selected $\Omega^{*}$, i.e., $\forall x\in\mathbb{R}^{dN}\setminus \mathcal{N}(\Psi^g)$, there is $x_{\pi^{*}}\in \Omega^{*}$ and $\Psi^g(x_{\pi^{*}})>0$, where $x_{\pi^{*}}$ is the input $x$ permuted by $\pi^{*}$. We use $\sigma(\pi^{*})$ denote the sorting sign of $\pi^{*}$ and $\pi^{*}(x,j)$ denote the index of $x_{j}$ after permutation. Then we let $\phi_{B}$ to satisfy that $\forall x\in\mathbb{R}^{dN}\setminus \mathcal{N}(\Psi^g)$, there is
\begin{equation}
    \phi_{B}(x_{j};\{x_{\setminus j}\})=\pi^{*}(x,j)\cdot\left(\frac{\Psi^{g}(x_{\pi^{*}})}{\prod_{i<j}(j-i)}\right)^{\frac{2}{N(N-1)}}\textrm{;}
    \label{equ43}
\end{equation}
and $\forall x\in\mathcal{N}(\Psi^g)$, there is $\phi_{B}(x_{j};\{x_{\setminus j}\})=0$, where $j=1,\cdots,N$. Note that the tiling property also claims that if two points $x$ and $x'$ come from the same nodal cell (i.e., $\Omega(\Psi^{g},x)=\Omega(\Psi^{g},x')$), then $\pi^*$ executes the same permutation operation on them (i.e., $\pi^{*}(x,j)=\pi^{*}(x',j)$ for any index $j$). Thus, the function $\pi^{*}(x,j)$ is constant inside any nodal cell, and due to the continuity of $\Psi^{g}$, we know that $\phi_{B}$ is continuous on $\mathbb{R}^{dN}\setminus \mathcal{N}(\Psi^g)$. Furthermore, when the input $x$ approaches to the node set $\mathcal{N}(\Psi^g)$, the value of $\Psi^{g}(x_{\pi^{*}})$ will consistently approach to zero. So when $x\rightarrow \mathcal{N}(\Psi^g)$, there is $\phi_{B}(x_{j};\{x_{\setminus j}\})\rightarrow 0$ for any index $j$, by which we can conclude that $\phi_{B}$ is continuous on the entire space of $\mathbb{R}^{dN}$. Using the construction of continuous $\phi_{B}$ in Eq.~(\ref{equ43}), we have
\begin{equation*}
    \begin{split}
        \psi'_{\textrm{pair}}(x)=&\prod_{i<j}\left(\phi_{B}(x_{j};\{x_{\setminus j}\})-\phi_{B}(x_{i};\{x_{\setminus i}\})\right)\\
        =&\sigma(\pi^{*})\cdot\prod_{i<j}(j-i)\cdot\left(\frac{\Psi(x_{\pi^{*}})}{\prod_{i<j}(j-i)}\right)^{\frac{2}{N(N-1)}}\\
        =&\sigma(\pi^{*})\cdot\Psi(x_{\pi^{*}})=\sigma(\pi^{*})\cdot\sigma(\pi^{*})\cdot\Psi(x)=\Psi(x)\textrm{,}
    \end{split}
\end{equation*}
which holds for $\forall x$.
\end{proof}

As to the antisymmetric function $\Psi(x)$ that is not ground-state, there still exist $\phi_{B}$ (but may not be continuous) making $\psi'_{\textrm{pair}}(x)$ represent $\Psi$, as stated below:
\begin{corollary}
\label{corollary1}
(Universality) For any antisymmetric function $\Psi(x)$, there exist a multi-electron function $\phi_{B}$, such that $\forall x$, $\psi'_{\textrm{pair}}(x)=\Psi(x)$.
\end{corollary}
\begin{proof}
Similar to the proof routine of Appendix B in \citet{pfau2020ab}, given a certain permutation $\pi$, we use $\sigma(\pi)$ to denote the sorting sign of $\pi$. For notation compactness, we define $\Lambda$ as 
\begin{equation*}
    \Lambda=\prod_{3\leq j \leq N}j\left(j+\textrm{sgn}(\Psi(x_{\pi}))\right)\prod_{3\leq i<j\leq N}(j-i)\textrm{,}
\end{equation*}
where $\Psi(x)$ is the antisymmetric function that we aim to represent. It is easy to verify that $\Lambda>0$. Then we could let $\phi_{B}$ satisfy that
\begin{equation*}
    \begin{split}
        &\phi_{B}(x_{\pi_1};\{x_{\pi_{\setminus 1}}\})=-\textrm{sgn}(\Psi(x_{\pi}))\cdot\left(\frac{|\Psi(x_{\pi})|}{\Lambda}\right)^{\frac{2}{N(N-1)}}\textrm{;}\\
        &\phi_{B}(x_{\pi_2};\{x_{\pi_{\setminus 2}}\})=0\textrm{;}\\
        &\phi_{B}(x_{\pi_j};\{x_{\pi_{\setminus j}}\})=j\cdot\left(\frac{|\Psi(x_{\pi})|}{\Lambda}\right)^{\frac{2}{N(N-1)}}\!\!\!\textrm{, }\forall 3\leq j\leq N\textrm{,}
    \end{split}
\end{equation*}
where $x_{\pi_j}$ is the $j$-th element of $x_{\pi}$. It is easy to verify that
\begin{equation*}
    \begin{split}
        \psi'_{\textrm{pair}}(x)=&\prod_{i<j}\left(\phi_{B}(x_{j};\{x_{\setminus j}\})-\phi_{B}(x_{i};\{x_{\setminus i}\})\right)\\
        =&\sigma(\pi)\cdot\Lambda\cdot\textrm{sgn}(\Psi(x_{\pi}))\cdot\frac{|\Psi(x_{\pi})|}{\Lambda}\\
        =&\sigma(\pi)\cdot\Psi(x_{\pi})=\sigma(\pi)\cdot\sigma(\pi)\cdot\Psi(x)=\Psi(x)\textrm{,}
    \end{split}
\end{equation*}
which holds for $\forall x$.
\end{proof}

\textbf{Connection with $\psi_{\textrm{Fermi}}(x)$.} The form of $\psi'_{\textrm{pair}}(x)$ can be written as the determinant of a Vandermonde matrix
\begin{equation*}
    \psi'_{\textrm{pair}}(x)=\begin{vmatrix}
    1 \! & \! \cdots \! & \! 1 \\
    \phi_{B}(x_{1};\{x_{\setminus 1}\}) \!  & \! \cdots \! & \! \phi_{B}(x_{N};\{x_{\setminus N}\}) \\
    \phi^{2}_{B}(x_{1};\{x_{\setminus 1}\}) \!  & \! \cdots \! & \! \phi^{2}_{B}(x_{N};\{x_{\setminus N}\}) \\
     \vdots \! &  \! \vdots \! & \! \vdots \\
     \phi^{N-1}_{B}(x_{1};\{x_{\setminus 1}\}) \!  & \! \cdots \! & \! \phi^{N-1}_{B}(x_{N};\{x_{\setminus N}\})
    \end{vmatrix}\textrm{,}
\end{equation*}
which can be regarded as a special case of $\psi_{\textrm{Fermi}}(x)$ where $\phi_{i}=\phi_{B}^{i-1}$ (but $\psi'_{\textrm{pair}}(x)$ only requires $O(N^2)$ computation). Therefore, the model family of $\psi_{\textrm{Fermi}}(x)$ is strictly more powerful than $\psi'_{\textrm{pair}}(x)$, and consequently the continuous universality of representing ground-state wavefunctions (claimed in Theorem~\ref{theorem5}) also holds for $\psi_{\textrm{Fermi}}(x)$.

\textbf{Extension of $\psi'_{\textrm{pair}}(x)$.} We can extend $\psi'_{\textrm{pair}}(x)$ to involve more permutation-equivariant functions, e.g., we construct
\begin{equation}
    \psi''_{\textrm{pair}}(x)=\!\!\!\!\!\!\prod_{1\leq i<j\leq N}\!\begin{vmatrix} \phi_{A}(x_{i};\{x_{\setminus i}\}) &  \phi_{A}(x_{j};\{x_{\setminus j}\}) \\
    \phi_{B}(x_{i};\{x_{\setminus i}\}) & \phi_{B}(x_{j};\{x_{\setminus j}\})
    \end{vmatrix}\textrm{.}
    \label{equ51}
\end{equation}
Specially, if we trivially choose $\forall x$, $\phi_{A}(x_{j};\{x_{\setminus j}\})=1$ be a constant function, then $\psi''_{\textrm{pair}}(x)$ degenerates to $\psi'_{\textrm{pair}}(x)$. Therefore, the model family of $\psi''_{\textrm{pair}}(x)$ is also strictly more powerful than $\psi'_{\textrm{pair}}(x)$, and the similar conclusions in Theorem~\ref{theorem5} and Corollary~\ref{corollary1} also hold for $\psi''_{\textrm{pair}}(x)$.

\label{complexity}
\textbf{Computational complexity of implementing} $\psi_{\textrm{pair}}(x)$. As analyzed in the Section \RomanNumeralCaps4 B of \citet{pfau2020ab}, the complexity of implementing the multi-electron orbitals $\phi_{1},\cdots,\phi_{N}$ is $O(N^2)$ for single-atom systems and fixed numbers of hidden units. After substituting the Slater determinant in FermiNet with our pairwise constructions (as described in Eq.~\eqref{equ50} and Eq.~\eqref{equ51}), we know that the cost of computing $\phi_{A}$ and $\phi_{B}$ is no more than $O(N^2)$. Besides, there are $C_{N}^{2}=N\times(N-1)/2$ terms and each term $\mathcal{A}\circ F(x_i,x_j)$ requires $O(1)$ extra computation on top of $\phi_{B}$ (and $\phi_{A}$). Therefore, totally $\psi'_{\textrm{pair}}(x)$ and $\psi''_{\textrm{pair}}(x)$ still require $O(N^2)$ computation.

\begin{table}[t]
  \centering
  \footnotesize
  \vspace{-0.2cm}
  \caption{Ground state energy. The values of `Exact' column come from \citet{chakravorty1993ground}. The values of `$\psi'_{\textrm{pair}}$' and `$\psi''_{\textrm{pair}}$' are averaged on the last 1,000 iterations with sampling stride of 10.}
  \vspace{0.25cm}
  \renewcommand*{\arraystretch}{1.2}
    \begin{tabular}{llll}
    \toprule
Atom & $\psi''_{\textrm{pair}}$ & $\psi'_{\textrm{pair}}$ & Exact   \\
\midrule
 Li & $-7.4782$ & $-7.4781$ & $-7.47806032$ \\
 Be & $-14.6673$ & $-14.6664$ & $-14.66736$ \\
 B & $-24.5602$ & $-24.4475$ & $-24.65391$ \\
 C & $-37.3531$ & $-37.2785$ & $-37.8450$ \\
 N & $-53.1855$ & $-53.0626$ & $-54.5892$ \\
    \bottomrule
    \end{tabular}%
    \label{table1}
\end{table}%

\subsection{Empirical Evaluation}
In the experiments, we follow the main training settings in~\citet{pfau2020ab}, as briefly recapped below. The FermiNet architecture includes four layers, each having 256 hidden units for the one-electron stream and 32 hidden units for the two-electron stream. The non-linear activation function is Tanh. To empirically estimate the gradient, 2,048 samples drawn from $\Psi_{\textrm{Fermi}}(x)$ are used in each training batch, with ensemble number $K=16$. These samples are updated via the random-walk MCMC algorithm, with 10 iterations and a move width of $0.02$. The optimizer is a modified version of KFAC~\citep{martens2015optimizing}, with the initial learning rate of $10^{-4}$.


In Table~\ref{table1}, we show some initial results using our pairwise constructions, where we substitute $\Psi_{\textrm{Fermi}}(x)$ with $\Psi'_{\textrm{pair}}(x)$ and $\Psi''_{\textrm{pair}}(x)$, respectively. Here we do not apply Hartree-Fock initialization. Our computational cost is less than FermiNet from two aspects: first, we use only one or two permutation-equivariant functions $\phi_{A}$ and $\phi_{B}$ for different atoms, while the number of functions $\phi_{1},\cdots,\phi_{N}$ used in FermiNet scales with $N$; second, we do not involve determinant operators and only require $O(N^2)$ computation. However, as seen in the results, our methods empirically ask for permutation-equivariant functions with higher model capacity for larger atoms. This can be achieved via applying more powerful network architectures, and we leave this to future exploration.

\section{Conclusion}
This paper develops a simple way toward $O(N^{2})$ universal antisymmetry in FermiNet by pairwise constructions. Our methods can also seamlessly adapt to other determinant-based Ansatz like PauliNet~\citep{hermann2020deep}. The universality of our pairwise Ansatz is established on the employment of multi-electron neural networks, which are proved to be universal approximators~\citep{hornik1989multilayer}. As to ground-state wavefunctions, we prove the existence of continuous construction, which facilitates practical learning via finite-capacity networks. 

Although theoretically, we demonstrate the (continuous) universality of pairwise constructions, empirically, we may require a large model capacity to fit the ground-truth $\phi_{A}$, $\phi_{B}$. This limitation is reflected in our initial experiments, as seen in Table~\ref{table1}. Nevertheless, the model architecture used in FermiNet is still quite shallow from the aspect of deep learning, and many model designs could be applied to obtain much more powerful multi-electron functions. Another limitation is that there is no clear advantage compared to the Slater determinant when $N$ is small, because the neural network part dominates the computation in this case.


\bibliographystyle{plainnat}
\bibliography{main}

\end{document}